\documentclass[twoside]{article}

\usepackage[accepted]{aistats2027}
\usepackage{amsmath,amssymb,amsthm,mathtools}
\usepackage{graphicx,booktabs,xcolor}

\usepackage{booktabs}
\usepackage{multirow}

\theoremstyle{plain}
\newtheorem{theorem}{Theorem}[section]
\newtheorem{proposition}[theorem]{Proposition}
\newtheorem{lemma}[theorem]{Lemma}
\newtheorem{corollary}[theorem]{Corollary}
\theoremstyle{definition}

\begin{document}

%

%

\twocolumn[

\aistatstitle{Efficient quadratic entropy with distance sketches}

\aistatsauthor{ Steve Huntsman }

\aistatsaddress{ steve.huntsman@cynnovative.com } ]

\begin{abstract}
We detail scalable methods for approximating the quadratic entropy $p^T d p$ for arbitrary distributions $p$ and common distances $d$ of negative type. We focus on the Euclidean and spherical geodesic cases, which both use random feature embeddings and projections to dramatically improve computational complexity within a simple framework. Amortization of a single large matrix multiplication and control variates further enable computation at large scale with low memory and runtime in situations where $d$ is held constant while $p$ varies. We demonstrate this with a comparison against direct pair sampling and bibliometric/scientometric examples on Open Graph Benchmark datasets, revealing papers, fields, and institutions with both particularly narrow and broad interdisciplinary reach from their citations and text features alone.
\end{abstract}

\section{Introduction}\label{sec:Introduction}

We give scalable unbiased Monte Carlo estimators for the quadratic entropy \cite{rao1984convexity} 
\begin{equation}\label{eq:quadraticEntropy}
Q := p^T d p = \mathbb{E}_{x,x'\sim p}d(x,x')
\end{equation}
for common distances $d$ of negative type. We target the regime of millions of arbitrary distributions $p$ on the same $n$ points represented in $\mathbb{R}^m$ (hence the same $d$) for $m$ in the hundreds or more and $n$ in the hundreds of thousands or more. This regime makes explicitly forming $d$ prohibitive, but it also encourages amortization of calculations. 


\subsection{Context}\label{sec:context}

The context we are primarily concerned with is that of a Markov chain with transition matrix $P$ whose rows $p^T$ we treat as distributions to compare in light of an underlying distance $d$ (perhaps not a metric \emph{per se}, as with squared Euclidean or cosine distance) of negative type (see \S \ref{sec:negativeType}). \footnote{Multiple distributions on the same distance geometry also arise in (e.g.) ecological contexts without necessarily packaging distributions into a Markov chain.} Directed graphs with node features (e.g., the node property Open Graph Benchmark examples \cite{hu2020open}) naturally fit into this context by viewing the uniform random walk as a Markov chain. We produce examples along these lines in \S \ref{sec:experiment} that illustrate how quadratic entropy captures semantic dispersion of citations.

\subsection{Related work}\label{sec:related}

The quadratic entropy \eqref{eq:quadraticEntropy} is a self-interaction functional that is related to two-sample maximum mean discrepancy and energy distance comparison functionals of the form $(p'-p)^T d(p'-p)$ \cite{gretton2012kernel,szekely2013energy,sejdinovic2013equivalence,lyons2013distance}. Quadratic entropy is also a particularly important case of more general notions of diversity \cite{leinster2012measuring,meckes2013positive,leinster2021entropy,huntsman2025peeling}; this framing motivates our constructions, but the relevant literature is largely unconcerned with computation at scale. The simplest approach to approximating \eqref{eq:quadraticEntropy} is accumulating entries of $d$ sampled from $p \times p$. The resulting variance converges at the same $O(1/N)$ rate, but with a much worse constant when $p$ is spread out, because it does not aggregate information across entries of $d$ within a sample (see the right panel of Figure \ref{fig:varianceRatio}). While it can give better results for sufficiently concentrated data (as in the left panel of Figure \ref{fig:varianceRatio}), it cannot amortize across different distributions $p$, and does not express $Q$ through a functional linear in $p$, so it cannot be effectively deployed at scale such as in \S \ref{sec:experiment}.

Random Fourier features \cite{rahimi2007random,zhao2015fastmmd} require a shift-invariant positive definite kernel to compute quantities like \eqref{eq:quadraticEntropy}, which precludes applications to Euclidean distance by Bochner's theorem and non-integrability (cf. \cite{hertrich2025fast}). On the other hand, slicing does apply: our treatment of the Euclidean case in \S \ref{sec:Euclidean} is included mainly for the sake of a joint framework with the spherical geodesic case in \S \ref{sec:geodesic}, but the Euclidean capability \emph{per se} is largely covered by the slicing and fast summation literature \cite{bonneel2015sliced,kolouri2020generalized,hertrich2024fast,hertrich2025fast,hertrich2024generative}, including use of an unweighted form in \cite{hertrich2024generative} of the classical Gini identity Proposition \ref{pro:GiniMeanDifference}. This gives one-dimensional negative distance kernel sums in $O(n\log n)$, and slicing extends to any radial kernel with an analytic basis function \cite{hertrich2024fast}, albeit via fast Fourier summation instead of sorting. Meanwhile, quasi-Monte Carlo directions can improve slicing performance but not in our regime, due to nonsmoothness \cite{hertrich2025fast} and high dimension. 

The work closest to \S \ref{sec:Euclidean} is \cite{boufadene2025fast}, which establishes a sliced identity for the negative Euclidean distance kernel with a Bernstein-type error bound, and which yields our estimator straightforwardly. Our treatment of the Euclidean case is nevertheless distinguished from existing work by a relative error bound on $Q$ (Theorem \ref{thm:EuclideanEstimator}), and by a squared-Euclidean control variate with mean computable in $O(mn)$ (\S \ref{sec:controlVariates}).

By comparison, \S \ref{sec:geodesic} reduces the spherical geodesic quadratic entropy to a single scalar per random direction, while remaining connected with the Euclidean case. The natural kernel is the zeroth-order arc-cosine kernel of \cite{cho2009kernel}: i.e., the affine transformation $\kappa_0 = 1 - \theta/\pi$ of the geodesic distance $\theta$. This kernel's random features are the signs of random projections rather than Fourier features of the sort in \cite{rahimi2007random}: these signs realize the affine transformation $2\kappa_0 - 1 = 1 - 2\theta/\pi$ \cite{goemans1995improved} familiar from locality-sensitive hashing \cite{charikar2002similarity}. Positive definiteness is not an obstruction here: the kernel of \cite{cho2009kernel} is positive definite on $S^{m-1}$ by Schoenberg's characterization \cite{schoenberg1942positive}, as well as by the Gram matrix representation in \S \ref{sec:geodesic}. The obstructions are instead the unavailability of a spectral measure that is easy to sample from, and that the spherical harmonics involved are infeasible to work with in high dimension. A deterministic alternative uses Corollary \ref{cor:grothendieck} to compute $Q$ using an entrywise expansion in powers of the Gram matrix, but the resulting matrix rank grows prohibitively \cite{scholkopf2002learning}. 
Instead, the applicable tool is the Goemans-Williamson representation \eqref{eq:GoemansWilliamson} of spherical geodesic distance, which through a \emph{sketch} given by the product $U^T X$ of IID $\mathcal{N}(0,I_m)$ random variables in $U$ and data points in $X$, enables the amortized calculation of $s(u) = p^T \mathrm{sgn}(X^T u)$ that underlies our development. Slicing on the sphere has been used for optimal transport \cite{bonet2023spherical,quellmalz2023sliced}, but those constructions project onto great circles to compare measures, rather than linearizing the geodesic distance.

A single cached sketch $U^T X$ is the key to computing \eqref{eq:quadraticEntropy} for both Euclidean and spherical geodesic cases while amortizing over instances of $p$. 
Structured accelerations of the arc-cosine family use a similar matrix product \cite{yu2016orthogonal,choromanski2017unreasonable,lyu2017spherical,munkhoeva2018quadrature} but do not bear on our results: \cite{lyu2017spherical,munkhoeva2018quadrature} tie the feature budget $N$ to the ambient dimension $m$ and supply no finite-sample guarantee for $N \ll m$, and the orthogonal random feature constructions of \cite{yu2016orthogonal,choromanski2017unreasonable} confer limited to no advantage at our scale, where near-orthogonality is practically free. A benefit of our estimators is error bounds on quadratic entropy that are independent of ambient dimension $m$. These improve on direction counts that scale linearly in $m$ in \cite{munkhoeva2018quadrature,boufadene2025fast}. Part of this gap is attributable to normalization, since directions drawn uniformly from the sphere carry a factor $\sqrt{\pi}\,\Gamma(\tfrac{m+1}{2})/\Gamma(\tfrac{m}{2}) \sim \sqrt{\pi m/2}$ absent for Gaussian $u$;
the qualitative independence from $m$ is nonetheless not shared by any bound we compare against.
Meanwhile, {Nystr\"om} methods \cite{williams2000using,yang2012nystrom} apply to Euclidean and spherical geodesic distances, but these methods are biased and have unreliable spectral norm error \cite{drineas2005nystrom,gittens2013revisiting}.

Finally, the methods here can be used to scale the \emph{optimization} of quadratic entropy, which is useful in its own right \cite{huntsman2025peeling}. This will be reported elsewhere.

\subsection{Paper organization}\label{sec:organization}

The remainder of this paper is as follows. In \S \ref{sec:negativeType} we discuss negative type distances and feature embeddings; \S \ref{sec:computations} discusses efficient estimators for quadratic entropy; \S \ref{sec:scaling} details scaling techniques, and \S \ref{sec:experiment} describes experiments on real data. 
Appendix \ref{sec:proofs} contains proofs. The LaTeX comments of the source file on arXiv contain code necessary for exact reproduction of our results.

\section{Negative type distances and feature embeddings}\label{sec:negativeType}

Let the columns of $X \in \mathbb{R}^{m \times n}$ represent $n$ points $x^{(j)}$ in $\mathbb{R}^m$ for $j \in [n]$, i.e., $x^{(j)}_i := X_{ij}$. Recall that a symmetric nonnegative matrix $d \in \mathbb{R}^{n \times n}$ with zero diagonal is called \emph{negative type} if $a^T d a \le 0$ for all $a \in \mathbb{R}^n$ with $1^T a = 0$; $d$ is called \emph{strict negative type} if the strict inequality $a^T d a < 0$ holds for $a \ne 0$. As indicated above, the distances we shall consider are all of the form $d_{jk} \equiv d(x^{(j)},x^{(k)})$, and they are all negative type, although squared Euclidean distance and cosine distance (the latter a special case of the former) are not metrics \emph{per se}; meanwhile, Euclidean distance matrices are always strict negative type, and spherical geodesic distance matrices are strict negative type as long as no two points are antipodal \cite{hjorth1998finite}. 

In the finite setting, a classical result of Schoenberg \cite{schoenberg1938metric} is that $d$ is negative type iff there exist ``feature'' vectors $\psi^{(j)} \in \mathbb{R}^n$ such that $d_{jk} = \| \psi^{(j)} - \psi^{(k)} \|^2$ identically. This works as follows: let $J := I - \frac{1}{n}11^T$ and set $G := -\frac{1}{2} J^T d J$. When $d$ is negative type (respectively, strict negative type) $G$ is positive semidefinite (respectively, positive definite), since $b^T G b = -\frac{1}{2} a^T d a$ with $a = Jb$ automatically satisfying $1^T a = 0$ for all $b$, so that $b^T G b \ge 0$ (respectively, $b^T G b > 0$ for $b \ne 0$). A line or two of algebra gives
\begin{equation}\label{eq:gram}
G_{jj} + G_{kk} - 2G_{jk} = d_{jk}
\end{equation}
and the Cholesky factorization $G = \Psi^T \Psi$ is such that the columns $\psi^{(j)}$ of $\Psi$ are the desired feature vectors, so $G_{jk} = \psi^{(j)T} \psi^{(k)}$. Inserting this into \eqref{eq:gram} yields (with $1^T p = 1$ here and throughout the rest of the paper)
\begin{equation}\label{eq:variance}
p^T d p = 2 \left ( \sum\nolimits_j p_j \| \psi^{(j)} \|^2 - \left \| \sum\nolimits_j p_j \psi^{(j)} \right \|^2 \right ).
\end{equation}
The right side of \eqref{eq:variance} is cheaper to compute than the left side's form suggests. For the other direction, start from $d_{jk} = \| \psi^{(j)} - \psi^{(k)} \|^2$: then if $1^T a = 0$, it follows that $\sum\nolimits_{j,k} a_j a_k \| \psi^{(j)} - \psi^{(k)} \|^2 = -2\| \sum_j a_j \psi^{(j)}\|^2 \le 0$, so that $d$ is in fact negative type.

\section{Efficient quadratic entropy computations for common distances of negative type}\label{sec:computations}

Squared Euclidean (\S \ref{sec:squaredEuclidean}) and cosine (\S \ref{sec:cosine}) distances yield trivially efficient computations of quadratic entropy. For ordinary Euclidean (\S \ref{sec:Euclidean}) and spherical geodesic (\S \ref{sec:geodesic}) metrics, more finesse is required, but we can still do better than a naive evaluation of \eqref{eq:quadraticEntropy}.

\subsection{Squared Euclidean distance}\label{sec:squaredEuclidean}

This is trivial, with $\psi^{(j)} \equiv x^{(j)}$: if $d(x,y) = \left \| x - y \right \|^2$, 
\begin{equation}\label{eq:squaredEuclidean}
p^T d p = 2 \sum\nolimits_j p_j \|x^{(j)}\|^2 - 2 \left \| \sum\nolimits_j p_j x^{(j)} \right \|^2.
\end{equation}

\subsection{Cosine distance on $S^{m-1}$}\label{sec:cosine}

This is actually a special case of squared Euclidean distance, since $\frac{1}{2}(\hat x - \hat y)^T(\hat x - \hat y) = 1-\hat x^T \hat y$:
\begin{equation}\label{eq:cosine}
d(x,y) = 1-\hat x^T \hat y \Rightarrow p^T d p = 1 - \left \| \sum\nolimits_j p_j \hat x^{(j)} \right \|^2.
\end{equation}

\subsection{Euclidean distance}\label{sec:Euclidean}

\begin{proposition}\label{pro:EuclideanSketch}
For $u \sim \mathcal{N}(0,I_m)$ and $x \in \mathbb{R}^m$,
\begin{equation}\label{eq:EuclideanSketch}
    \| x \| = \sqrt{\frac{\pi}{2}} \cdot \mathbb{E}_u | u^T x |;
\end{equation}
and for $d_{jk} \equiv \| x^{(j)} - x^{(k)} \|$,
\begin{equation}\label{eq:EuclideanSketchQuadraticDiversity}
    p^T d p = \sqrt{\frac{\pi}{2}} \cdot \mathbb{E}_u \sum\nolimits_{j,k} p_j p_k | u^T x^{(j)} - u^T x^{(k)} |.
\end{equation}
\end{proposition}


The representation \eqref{eq:EuclideanSketch} underlies slicing \cite{bonneel2015sliced}, but our application differs. We produce an unbiased estimator of the exact quadratic entropy instead of a proxy distance between distributions.

A simple trick allows us to evaluate the \emph{mean absolute difference} on the right side of \eqref{eq:EuclideanSketchQuadraticDiversity} more efficiently:

\begin{proposition}[Gini]\label{pro:GiniMeanDifference}
If $w \in \mathbb{R}^n$ and $v_1 \le \dots \le v_n$, then $ \sum\nolimits_{j,k} w_j w_k |v_j - v_k| $ equals
\begin{equation}\label{eq:GiniMeanDifference}
    2\sum\nolimits_j w_j v_j \left ( 2 \left [ \sum\nolimits_{i=1}^{j-1} w_i \right ] + w_j - \sum\nolimits_{i=1}^n w_i \right ).
\end{equation}
\end{proposition}


Since sorting has computational complexity $O(n \log n)$, the right hand side of \eqref{eq:GiniMeanDifference} inherits this computational complexity instead of the naive $O(n^2)$.  With Proposition \ref{pro:EuclideanSketch}, this yields a very efficient algorithm for estimating $p^T d p$: sample $N$ IID vectors $u \sim \mathcal{N}(0,I_m)$ and compute the corresponding sample mean estimating the right hand side of \eqref{eq:EuclideanSketchQuadraticDiversity}.
\begin{figure*}[t]
  \centering
  \includegraphics[trim = 0mm 0mm 0mm 0mm, clip, width=\textwidth,keepaspectratio]{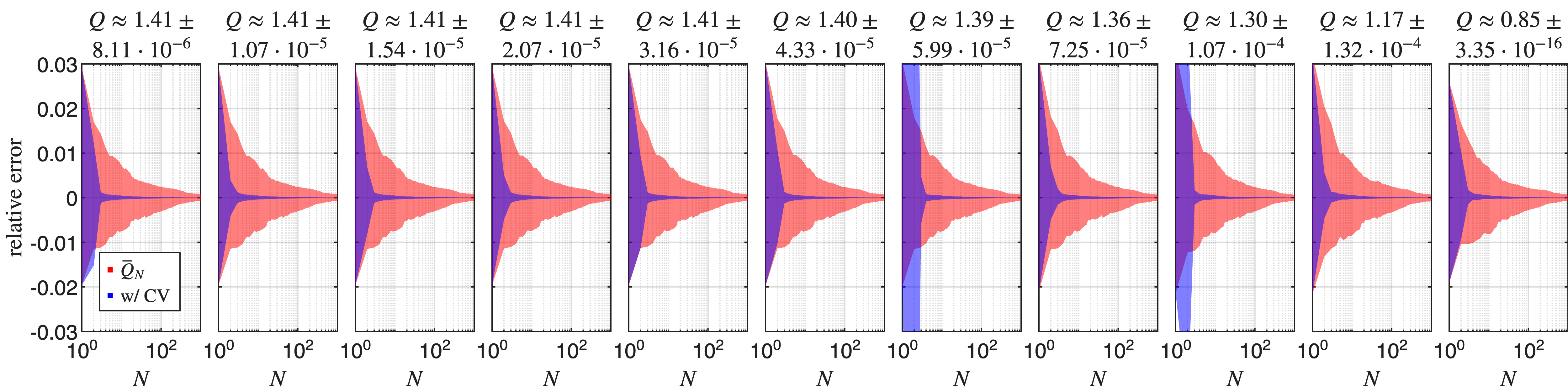}
  \caption{
  From left to right: for $\ell = 0,\dots,10$, relative errors over 100 trials for estimates of Euclidean quadratic entropies $\bar Q^\textnormal{euc}_N(p)$ with $p$ dense, given by normalizing $n$ IID $\mathcal{U}([0,1])$ random variables, and $d$ given by taking 10000 IID uniform points on $S^{m-1}$ for $m = 2^{10}$ and forcing $2^\ell$ coordinates (chosen uniformly without replacement) to be positive. {\color{red}Error for $\bar Q^\textnormal{euc}_N$ is shown in red}, while {\color{blue}error using the Euclidean control variate described in \S \ref{sec:controlVariates} is shown in blue}. Note that relative errors are virtually independent of the concentration parameter $\ell$.
  }
  \label{fig:fastQuadDivEuclidean}
\end{figure*}

Writing here 
\begin{align*}
q^\textnormal{euc}(u) & \equiv q^\textnormal{euc}(u;X,p) \\
& := \sqrt{\frac{\pi}{2}} \sum\nolimits_{j,k} p_j p_k |u^T x^{(j)} - u^T x^{(k)}|,
\end{align*}
the question now is how the estimator 
\begin{equation}\label{eq:EuclideanEstimator}
    \bar Q_N^\textnormal{euc} := \frac{1}{N} \sum_{i=1}^N q^\textnormal{euc}(u^{(i)})
\end{equation} behaves relative to $Q = p^T d p$. 
The following theorem provides an answer.

\begin{theorem}\label{thm:EuclideanEstimator}
    With IID $\mathcal{N}(0,I_m)$ inputs, $\bar Q_N^\textnormal{euc}$ is an unbiased estimator of $Q$, and
    \begin{equation}\label{eq:EuclideanEstimatorTail}
        \mathbb{P} \left ( \left | \frac{\bar Q_N^\textnormal{euc}}{Q} - 1 \right | \ge \varepsilon \right ) \le 2 \exp(-N \varepsilon^2/\pi),
    \end{equation}
    \begin{equation}\label{eq:EuclideanEstimatorVariance}
        \mathrm{Var}(\bar Q^\textnormal{euc}_N) \le \left ( \frac{\pi}{2} - 1 \right ) \frac{Q^2}{N}.
    \end{equation}
\end{theorem}

By Theorem \ref{thm:EuclideanEstimator}, a relative error $\varepsilon$ can be achieved with probability $1-\delta$ if 
\begin{equation}\label{eq:EuclideanEstimatorBudget}
    \frac{\pi}{\varepsilon^2} \log \frac{2}{\delta} \le N.
\end{equation}
For example, $\delta = 10^{-6}$ and $\varepsilon = \frac{1}{2}$ gives $N \ge 183$; while $\delta = 10^{-3}$ and $\varepsilon = \frac{1}{10}$ gives $N \ge 2388$. Figure \ref{fig:fastQuadDivEuclidean} and the variance bound \eqref{eq:EuclideanEstimatorVariance} both suggest that the exponential tail bound is quite loose in practice, and that small $N$ suffices for meaningful estimates.

Meanwhile, the bound \eqref{eq:EuclideanEstimatorVariance} is tight, saturating for $n = 2$ points with $p$ uniform.

\subsection{Geodesic distance on $S^{m-1}$}\label{sec:geodesic}

Instead of projecting and transporting, the spherical geodesic distance $\theta$ is linearized through random sign features that yield \eqref{eq:geodesicSketchQuadraticDiversity} below.
Write $s_j(u) := \text{sgn}(u^T x^{(j)}).$
The key technical element we need is Lemma 3.2 of \cite{goemans1995improved} (also Lemma 1.4.1 of \cite{gartner2012approximation}; cf. \cite{charikar2002similarity}), restated here:
\begin{lemma}[Goemans-Williamson]\label{lem:GoemansWilliamson}
If $u/\|u\| \sim \mathcal{U}(S^{m-1})$ (e.g., $u \sim \mathcal{N}(0,I_m)$), then
\begin{equation}\label{eq:GoemansWilliamson}
    \pi \cdot \mathbb{P} \left ( s_j(u) \ne s_k(u) \right ) = \arccos ( \hat x^{(j)T} \hat x^{(k)} ) =: \theta_{jk}.
\end{equation}
\end{lemma}

Taking expectations of $1-\delta_{jk} = \frac{1}{2}(1-s_j(u)s_k(u))$ now gives that 
\begin{equation}\label{eq:geodesic1}
    \theta_{jk} = \frac{\pi}{2} \mathbb{E}_u(1-s_j(u)s_k(u)).
\end{equation}

Let $\psi^{(j)} := \sqrt{\pi} \, s_j(\hat u)/2$ for $\hat u \in S^{m-1}$.
\begin{proposition}\label{pro:sphereSignEmbedding}
    $\| \psi^{(j)} - \psi^{(k)} \|_{L^2(S^{m-1})}^2 = \theta_{jk}$. 
\end{proposition}


\begin{corollary}[Grothendieck\footnote{See, e.g., \S 2 of \cite{briet2010grothendieck}.}]\label{cor:grothendieck}
    $\mathbb{E}_u(s_js_k) = \frac{2}{\pi} \arcsin( \hat x^{(j)T} \hat x^{(k)})$.
\end{corollary}

Now define $s(u) := \sum_j p_j s_j(u) \in [-1,1]$.

\begin{theorem}\label{thm:geodesicSketchQuadraticDiversity}
    For $d$ a geodesic distance matrix on $S^{m-1}$ and $u \sim \mathcal{N}(0,I_m)$,
    \begin{equation}\label{eq:geodesicSketchQuadraticDiversity}
        p^T d p = \frac{\pi}{2} \left (1 - \mathbb{E}_u[s^2] \right ).
    \end{equation}
\end{theorem}


Note that this result expresses the geodesic quadratic entropy very succinctly, through the scalar $s$ evaluated over random directions. This scalar can be efficiently amortized as discussed in \S \ref{sec:amortization} when $d$ is reused. Meanwhile, as in \S \ref{sec:Euclidean}, we obtain a very efficient algorithm for estimating $p^T d p$: sample $N$ IID vectors $u \sim \mathcal{N}(0,I_m)$ and compute the corresponding sample mean estimating the right hand side of \eqref{eq:geodesicSketchQuadraticDiversity}.
Figure \ref{fig:fastQuadDivGeodesic} shows an example using the algorithm.

\begin{figure*}[t]
  \centering
  \includegraphics[trim = 0mm 0mm 0mm 0mm, clip, width=\textwidth,keepaspectratio]{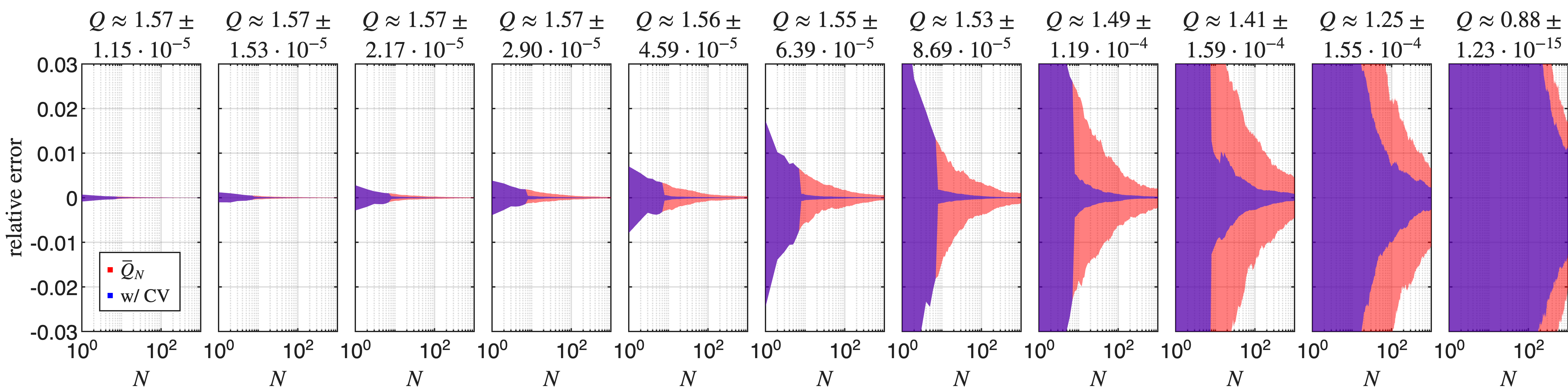}
  \caption{As in Figure \ref{fig:fastQuadDivEuclidean}, but for spherical geodesic (\emph{versus} Euclidean) distance $d$. 
  Note that as concentration increases (because the number $2^\ell$ of fixed signs on the sphere increases) from left ($\ell=0$) to right ($\ell=10$), the relative errors rapidly increase.   
  }
  \label{fig:fastQuadDivGeodesic}
\end{figure*}


Also as in \S \ref{sec:Euclidean}, if we write here 
$$q^\textnormal{geo}(u) \equiv q^\textnormal{geo}(u;X,p) := \frac{\pi}{2} \left ( 1-  s(u)^2 \right ),$$ 
the question now is how the estimator 
\begin{equation}\label{eq:geodesicEstimator}
    \bar Q_N^\textnormal{geo} := \frac{1}{N} \sum_{i=1}^N q^\textnormal{geo}(u^{(i)})
\end{equation} behaves relative to $Q := p^T d p$. The following theorem provides an answer.

\begin{theorem}\label{thm:geodesicEstimatorAbsolute}
    For $d$ a geodesic distance matrix on $S^{m-1}$ and with IID $\mathcal{N}(0,I_m)$ inputs, $\bar Q_N^\textnormal{geo}$ is an unbiased estimator of $Q$ with
    \begin{equation}\label{eq:geodesicEstimatorTail}
        \mathbb{P} \left ( \left | \bar Q_N^\textnormal{geo} - Q \right | \ge \varepsilon \right ) \le 2 \exp(-8N \varepsilon^2/\pi^2)
    \end{equation}
    and 
    \begin{equation}\label{eq:geodesicEstimatorVariance}
        \mathrm{Var}(\bar Q_N^\textnormal{geo}) \le \frac{\pi^2}{16N}.
    \end{equation}
\end{theorem}


By Theorem \ref{thm:geodesicEstimatorAbsolute}, an absolute error $\varepsilon$ can be achieved with probability $1-\delta$ if 
\begin{equation}\label{eq:geodesicEstimatorBudget}
    \frac{\pi^2}{8\varepsilon^2} \log \frac{2}{\delta} \le N.
\end{equation}
For example, $\delta = 10^{-6}$ and $\varepsilon = \frac{1}{2}$ gives $N \ge 72$; while $\delta = 10^{-3}$ and $\varepsilon = \frac{1}{10}$ gives $N \ge 938$. As before, Figure \ref{fig:fastQuadDivGeodesic} and the variance bound \eqref{eq:geodesicEstimatorVariance} both suggest that the exponential tail bound is quite loose in practice, and that taking $N$ small suffices for meaningful estimates.

\subsubsection{A relative bound}\label{sec:geodesicRelative}

We can also get a relative bound on $Q$ in the geodesic case, starting with a tightened variance bound. We first collect building blocks.

\begin{proposition}\label{prop:sharpvar}
Under the same assumptions as Theorem \ref{thm:geodesicEstimatorAbsolute},
\begin{equation}\label{eq:sharpvar}
  \mathrm{Var}\bigl(\bar{Q}^{\textnormal{geo}}_N\bigr) \;\le\; \frac{1}{N}\left(\frac{\pi Q}{2} - Q^2\right)
  \;\le\; \frac{\pi^2}{16N}.
\end{equation}
\end{proposition}


The bound \eqref{eq:sharpvar} is tight. To see this, take $n = 2$ points at geodesic distance $\theta$ with $p$ uniform: then $s = \tfrac{1}{2}(s_1 + s_2) \in \{0,\pm 1\}$, so $s^2$ is Bernoulli with $\mathbb{P}(s^2 = 1) = \mathbb{P}(s_1 = s_2) = 1 - \theta/\pi$ by Lemma \ref{lem:GoemansWilliamson}. Since $Q = \theta/2$ here, $\mathrm{Var}(s^2) = (1 - 2Q/\pi) \cdot 2Q/\pi$, and the inequalities in the proof become equalities. For $Q \ll 1$, the bound \eqref{eq:sharpvar} is also much sharper than \eqref{eq:geodesicEstimatorVariance}.

\begin{theorem}[Bernstein; adapted from Theorem 2.10 of \cite{boucheron2013concentration}]\label{thm:blm}
Let $\zeta_1,\dots,\zeta_N$ be IID real random variables with $\mathbb{E}\zeta_i = 0$, and suppose $\nu, c > 0$ satisfy
\begin{equation}\label{eq:bcond1}
  N \mathbb{E} \left[\zeta_i^2 \right] \le \nu
\end{equation}
and
\begin{equation}\label{eq:bcond2}
  N \mathbb{E}\left[\left(\zeta_i\right)_+^{\ell}\right]
   \le \frac{\ell!}{2} \nu c^{\ell-2},
\end{equation}
for every integer $\ell \ge 3$, where $(\cdot)_+ := \max\{\cdot,0\}$. Then for all $\tau > 0$,
\begin{equation}\label{eq:bernstein}
  \mathbb{P}\left(\sum_{i=1}^{N}\zeta_i \ge \sqrt{2\nu \tau} + c\tau\right) \le e^{-\tau}.
\end{equation}
\end{theorem}

\begin{lemma}\label{lem:bern}
Let $\zeta_1,\dots,\zeta_N$ be IID real random variables with $\mathbb{E}\zeta_i = 0$ and $|\zeta_i| \le R$, and let $\sigma^2 > 0$ satisfy $\mathbb{E}[\zeta_i^2] \le \sigma^2$. Then for every $r > 0$, \begin{equation}\label{eq:bernmean}
  \mathbb{P} \left( \left| \tfrac{1}{N} \textstyle\sum_{i=1}^{N}\zeta_i \right| \ge r \right)
  \le 2\exp\left( \frac{-N r^2}{2\sigma^2 + 2rR/3} \right).
\end{equation}
\end{lemma}

We are now ready to use Proposition \ref{prop:sharpvar} and Lemma \ref{lem:bern} for a relative bound.

\begin{theorem}\label{thm:geodesicEstimatorRelative}
Under the assumptions of Theorem \ref{thm:geodesicEstimatorAbsolute} with $Q \in (0,\pi/2)$,
\begin{equation}\label{eq:geobern}
  \mathbb{P}\left( \left| \frac{\bar{Q}^{\textnormal{geo}}_N}{Q} - 1 \right| \ge \varepsilon \right)
  \;\le\; 2\exp\left( \frac{-N \varepsilon^2 Q}{\pi\left(1 + \varepsilon/3\right)} \right).
\end{equation}
\end{theorem}

%

By Theorem \ref{thm:geodesicEstimatorRelative}, a relative error $\varepsilon$ is achieved with probability $1-\delta$ once 
\begin{equation}\label{eq:geobernN}
  N \ge \frac{\pi\left(1 + \varepsilon/3\right)}{\varepsilon^2 Q}\log\frac{2}{\delta}.
\end{equation}

\subsubsection{Comparison with a pair-sampling baseline}\label{sec:versusNaive}

Figure \ref{fig:varianceRatio} compares $q^\textnormal{geo}$ against pair sampling along two axes: the geometry concentration, as in Figure \ref{fig:fastQuadDivGeodesic}; and the support size of a uniform $p$. As the concentration  $\ell$ increases, the ratio falls quickly: the sketch variance increases towards the bound in Proposition \ref{prop:sharpvar}, while the pair variance, which measures the spread of individual distances, is agnostic to concentration. As $|\mathrm{supp}(p)|$ increases, the ratio rises quickly but is always above $1$ here. 
In short, concentrating the geometry or (the support of) $p$ erodes the advantage of sketching, but only the former reverses it at a useful scale.

\begin{figure*}[t]
  \centering
  \includegraphics[trim = 0mm 0mm 0mm 0mm, clip, width=\textwidth,keepaspectratio]{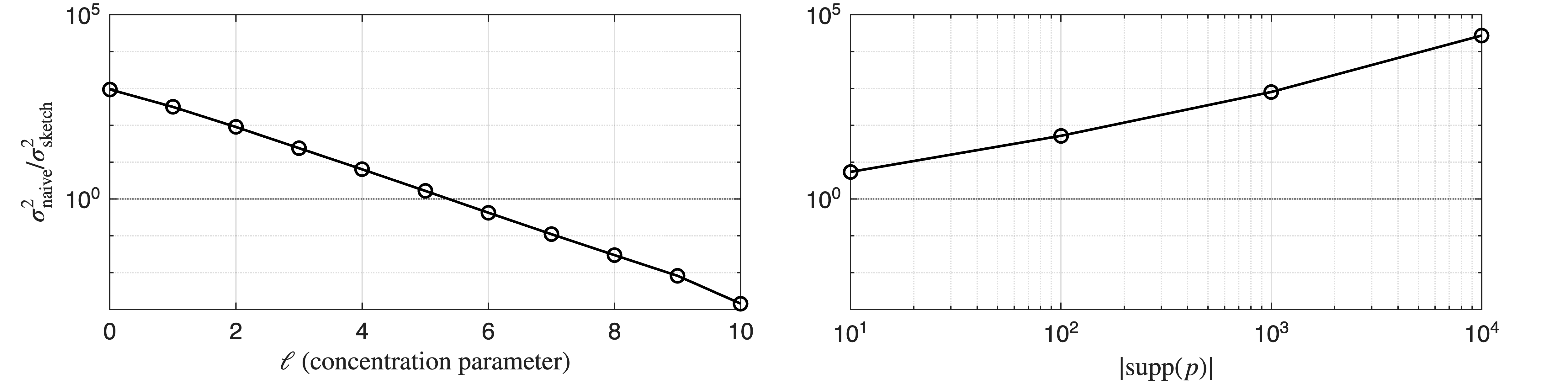}
  \caption{Both panels show the same ratio, of the per-pair variance $\sigma^2_\mathrm{naive} = p^T d^{\circ 2} p - Q^2$ (where we use the Hadamard product notation $(d^{\circ 2})_{jk} = d_{jk}^2$) of the pair-sampling estimator of \S \ref{sec:related} to the measured per-direction variance $\sigma^2_\mathrm{sketch} = \mathrm{Var}(q^\textnormal{geo}) = \frac{\pi^2}{4}\mathrm{Var}(s^2)$, for $n = 10000$ points on $S^{m-1}$ for $m = 2^{10}$. Left: $p$ dense and random, with the geometry concentrated by forcing $2^\ell$ coordinates positive as in Figure \ref{fig:fastQuadDivGeodesic}. Right: taking geometry and $p$ uniform on $\le n$ points. The panels do not agree at their common corner ($\ell = 0$ versus $|\textrm{supp}(p)| = 10000$) because $p$ is dense random on the left and exactly uniform on the right. }
  \label{fig:varianceRatio}
\end{figure*}

\section{Scaling}\label{sec:scaling}

The runtimes of the squared Euclidean, cosine distance, Euclidean, and spherical geodesic cases above are respectively $O(mn)$, $O(mn)$, $O((m+\log n)nN)$, and $O(mnN)$, each of which improves over the naive $O(mn^2)$ when as intended $N \ll n$. However, further speedups are achievable: amortization and a (Euclidean) control variate provide real advantage.

\subsection{Amortization}\label{sec:amortization}

Both Euclidean and geodesic distance estimators can (and should) reuse a single product $U^T X$ across different $p$, provided only that $X$ and $d$ are preserved. Similarly, the sort required for the Euclidean distance estimator can be done once and cached.


A naive approach requires a single $O(mn^2)$ effort to form $d$, followed by $O(n^2)$ effort for each $Q$, whereas our approach requires a single $O(mnN)$ effort to form $U^T X$ followed by $O(nN)$ for each $Q$.\footnote{
Sparsity in $p$ can further help our approach, though not uniquely, so we do not focus on this. In particular, sparsity makes the analysis of \S \ref{sec:experiment} lightweight, even for multiple random walk steps.
}

In the amortized setting, we can obtain uniform bounds on $Q$. 

\begin{lemma}\label{lem:uniformAmortizedEuclidean}
    Let $d$ be the Euclidean metric on $n$ distinct fixed points $x^{(j)} \in \mathbb{R}^m$ and $(\bar d_N)_{jk} := \sqrt{\frac{\pi}{2}} \frac{1}{N} \sum_{i=1}^N \left | u^{(i)T}(x^{(j)}-x^{(k)}) \right |$ for $u^{(i)} \sim \mathcal{N}(0,I_m)$ and $i \in [N]$. For     
    $$N \ge \frac{\pi}{\varepsilon^2} \log \frac{2\binom{n}{2}}{\delta},$$ 
    we have $\left | (\bar d_N)_{jk} - d_{jk} \right | \le \varepsilon d_{jk}$ for all $j,k$ with probability at least $1-\delta$; consequently, 
    $| \bar Q^\textnormal{euc}_N(p)/Q(p) - 1 | \le \varepsilon$ with probability at least $1-\delta$, independently of $p$.
\end{lemma}

A similar result for spherical geodesic distance is easier, but weaker:

\begin{lemma}\label{lem:uniformAmortizedGeodesic}
    Let $\theta$ be the spherical geodesic metric on $n$ distinct fixed points in $S^{m-1}$, satisfying $\theta_{jk} = \frac{\pi}{2} \mathbb{E}_u(1-s_j(u)s_k(u))$ for $u \sim \mathcal{N}(0,I)$. For $$N \ge \frac{\pi^2}{2\varepsilon^2} \log \frac{2\binom{n}{2}}{\delta},$$ 
    $\left | \bar Q^\textnormal{geo}_N(p) - Q(p) \right | \le \varepsilon \cdot (1-\|p\|^2) \le \varepsilon$ with probability $1-\delta$.
\end{lemma}



\subsection{Control variates}\label{sec:controlVariates}

Recall the idea of a \emph{control variate} \cite{asmussen2007stochastic}. Here, this means a random variable $W$ that is strongly correlated with $\bar Q_N^\bullet$ (here $\bullet$ is a placeholder indicating the Euclidean or spherical geodesic cases) and has a known mean $W_*$. In this context, consider the estimator $\bar Q_N^\bullet + \alpha(\bar W-W_*)$, where $\bar W$ is the sample mean of $W$. This estimator is also an unbiased estimator of $Q$, and its variance is $\text{Var}(\bar Q_N^\bullet) + \alpha^2 \cdot \text{Var}(\bar W) + 2 \alpha \cdot \text{Cov}(\bar Q_N^\bullet,\bar W)$. This variance is minimized for $\alpha = -\text{Cov}(\bar Q_N^\bullet, \bar W)/\text{Var}(\bar W)$, with minimum 
$$\text{Var}(\bar Q_N^\bullet) \cdot \left ( 1-\frac{\text{Cov}(\bar Q_N^\bullet,\bar W)^2}{\text{Var}(\bar Q_N^\bullet) \text{Var}(\bar W)} \right ).$$
The optimal $\alpha$ is estimated by sample (co)variances; leave-one-out cross-validation avoids a $O(1/N)$ bias that could be otherwise inherited from dependence on $\bar W$ and manifests in analogues of Figures \ref{fig:fastQuadDivEuclidean} and \ref{fig:fastQuadDivGeodesic} for large $\ell$ (not shown).

A useful control variate for the Euclidean distance quadratic entropy is 
\begin{align*}
w^\textnormal{euc}(u) & := \sum\nolimits_{j,k} p_j p_k (u^T x^{(j)} - u^T x^{(k)})^2 \\
& = 2 \sum\nolimits_j p_j (u^T x^{(j)})^2 - 2\left ( u^T \sum\nolimits_j p_j x^{(j)} \right )^2.
\end{align*}
To see this, note that 
$$\mathbb{E}_u([u^T x]^2) = \mathbb{E}_u(x^T u u^T x) = x^T \mathbb{E}_u (u u^T) x = \| x \|^2$$ 
and 
$$\mathbb{E}_u \left ( \left [u^T\sum\nolimits_j p_j x^{(j)} \right ]^2 \right ) = \left \| \sum\nolimits_j p_j x^{(j)} \right \|^2,$$
so that $\mathbb{E}_u w^\textnormal{euc}$ is given by the squared Euclidean quadratic entropy \eqref{eq:squaredEuclidean}, which is computable in $O(mn)$ \emph{versus} $O(mnN)$. Figure \ref{fig:fastQuadDivEuclidean} shows its application {\color{blue}in blue}.

A control variate for the geodesic distance case is 
$w^\textnormal{geo}(u) := ( \hat u^T \sum\nolimits_j p_j x^{(j)} )^2$,
which has
$\mathbb{E}_u w^\textnormal{geo} = \frac{1}{m} \| \sum\nolimits_j p_j x^{(j)} \|^2$. However, this is less effective for concentrated data than in the Euclidean case, because in that regime working with signs alone throws away too much information.
Figure \ref{fig:fastQuadDivGeodesic} shows its application {\color{blue}in blue}.

\section{Examples at scale}\label{sec:experiment}

As a demonstration/experiment of the sort familiar at smaller scales in the bibliometric and scientometric literature \cite{stirling2007general,porter2009science,cassi2017analysing,park2026interdisciplinary}, we compute quadratic entropies for a Markov chain corresponding to a uniform random walk of citations on \texttt{ogbn-arxiv}, a citation digraph of $n=169343$ arXiv papers with $1166243$ arcs, $m=128$ skip-gram title/abstract features per node, and a label for the arXiv subject class \cite{hu2020open}. We write $p_{(1)}^{(j)}$ for the one-step distribution of the $j$th paper, treated as a column vector; we place unit mass on $j$ for any papers without citations. We also consider the spherical geodesic metric on (normalized) features. If $\chi_j$ papers cite $j$ uniformly across subject classes, then the corresponding Shannon entropy is $H = \log \chi_j$, whereas $Q = (1-1/\chi_j)\bar \theta_j$, where $\bar \theta_j$ is the average distance between distinct citing papers: this measures dispersion, and not just cardinality. This also suggests a normalization that controls for citation counts and focuses on geometry.

Figure \ref{fig:dispersionNormalized} is the result of applying this normalization to the data in Figure \ref{fig:dispersion}. In both figures, we stratify the citation count $\chi$ into the bands $\chi \in [10,20)$, $[20,50)$, $[50,100)$, and $[100,500)$, so that the Shannon entropy $H$ is approximately constant within each band 
while we allow the quadratic entropy to vary. Meanwhile, we consider other proxies for the dispersion of papers' readership: the number of distinct arXiv subject areas represented among citing papers (the left panels of Figures \ref{fig:dispersion} and \ref{fig:dispersionNormalized}) and the Shannon entropy of the citing papers (the right panels of the same figures). We discard cells with fewer than $30$ papers. The figures show that the feature-based quadratic entropy does capture a ground-truth (defined by arXiv subject area) semantic dispersion of citing papers, while controlling for the citation count $\chi$. Finally, Table \ref{tab:reach-exemplars} shows some specific papers with particularly wide or narrow reach.

\begin{figure}[htbp]
  \centering
  \includegraphics[trim = 0mm 0mm 0mm 0mm, clip, width=\columnwidth,keepaspectratio]{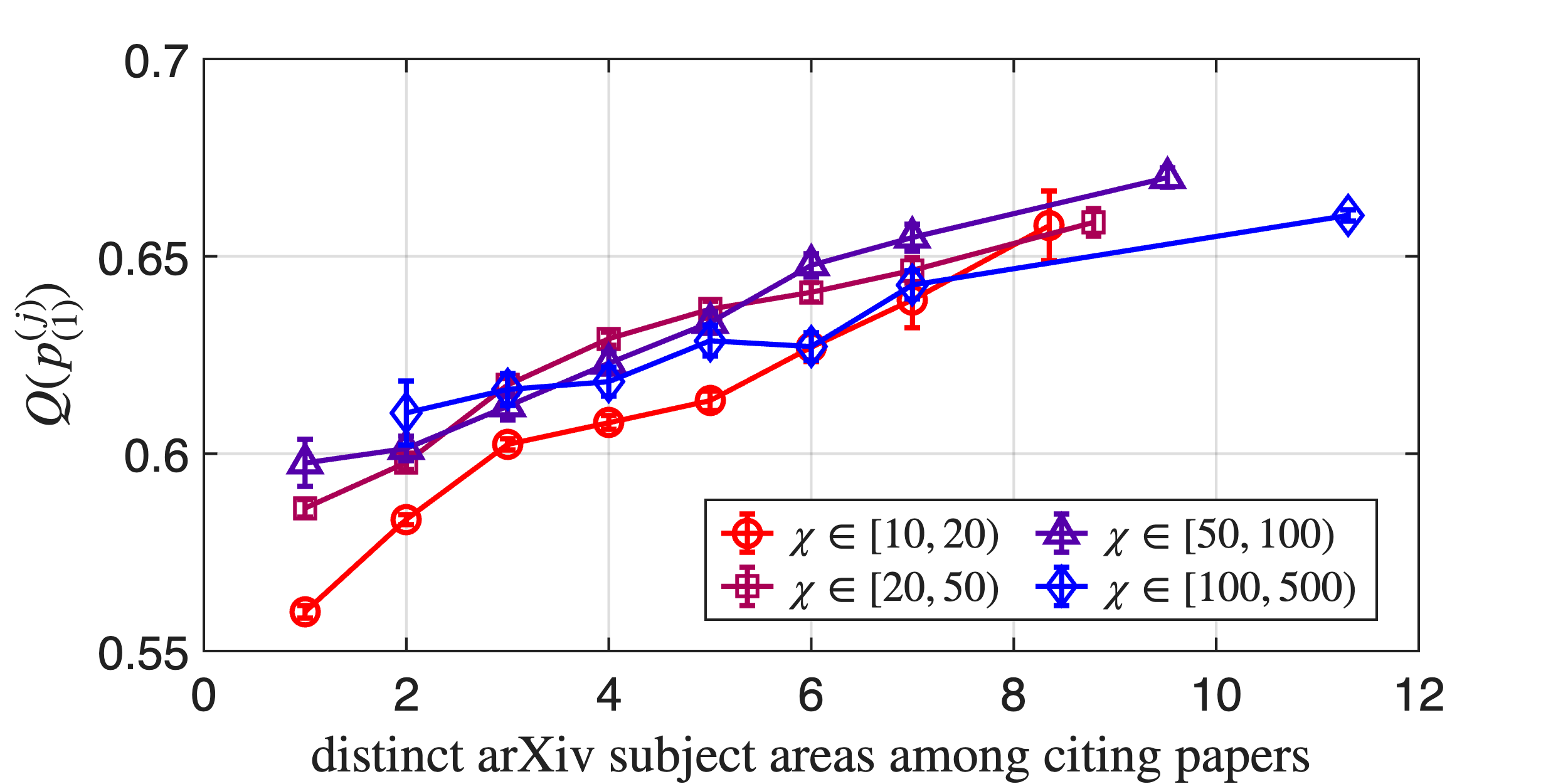}
  
  \includegraphics[trim = 0mm 0mm 0mm 0mm, clip, width=\columnwidth,keepaspectratio]{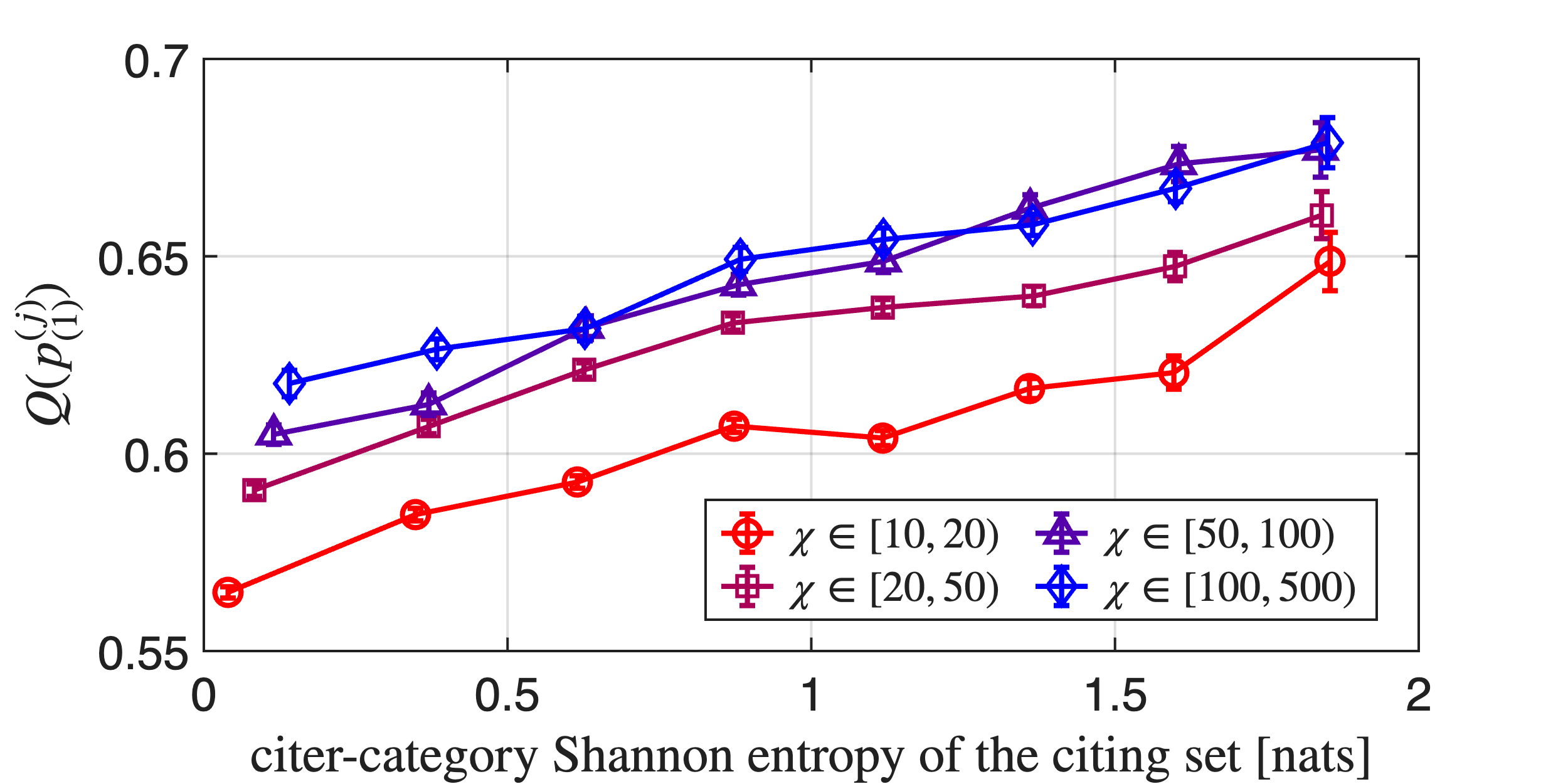}
  \caption{Top: (estimated) quadratic entropies of the (one-step) distribution $p^{(j)}_{(1)}$ of citing papers for various bands of citation count $\chi$ and as a function of distinct subject areas of citing papers. Bottom: the same quadratic entropies, but as functions of Shannon entropy for the arXiv class distribution. Note that $d$ is agnostic to the arXiv class data and only depends on the skip-gram feature vectors for titles and abstracts.}
  \label{fig:dispersion}
\end{figure}

\begin{figure}[htbp]
  \centering
  \includegraphics[trim = 0mm 0mm 0mm 0mm, clip, width=\columnwidth,keepaspectratio]{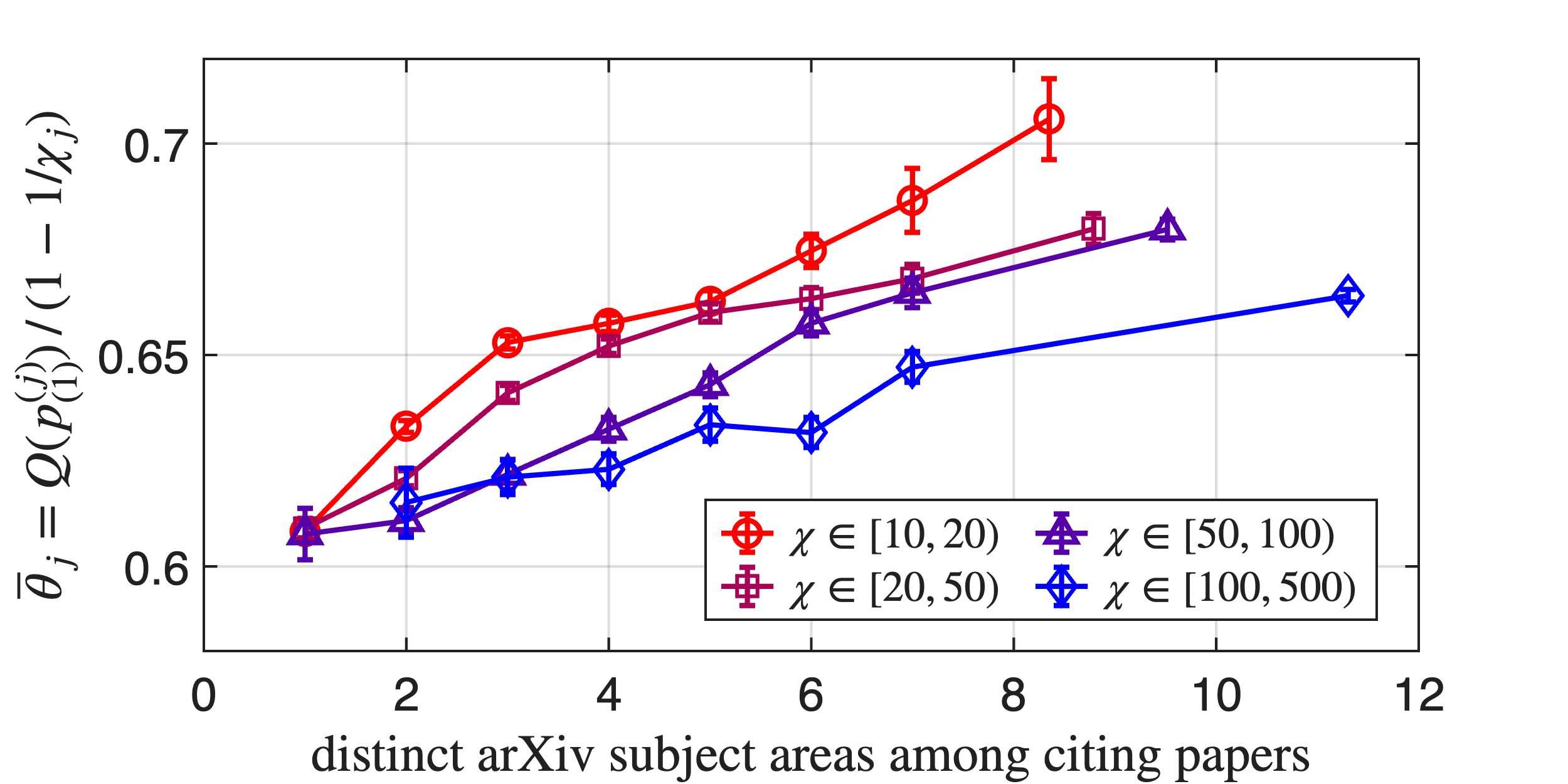}
  \includegraphics[trim = 0mm 0mm 0mm 0mm, clip, width=\columnwidth,keepaspectratio]{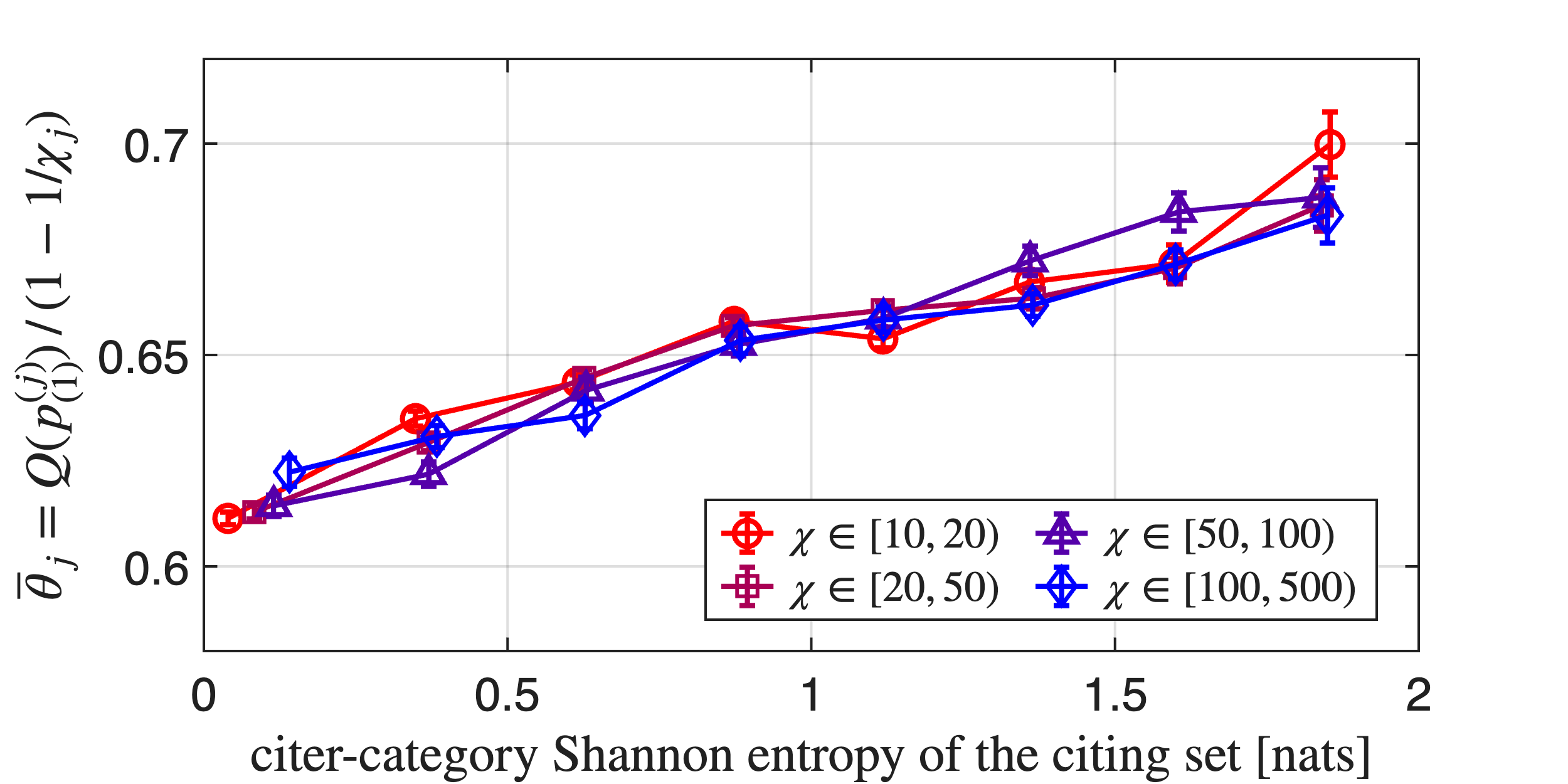}
  \caption{As in Figure \ref{fig:dispersion}, but normalized.}
  \label{fig:dispersionNormalized}
\end{figure}

\begin{table*}[t]
  \centering
  \caption{Narrow- and broad-reach exemplars at three target degrees. For each of two papers in the three target bands $[18,22)$, $[38,42)$, and $[75,85)$, we show the paper's degree $\chi$, the number $\#$ of distinct arXiv categories of citing papers, the cross-field fraction $f_{\times}$, the quadratic entropy $Q$, the primary class, arXiv identifier, and the beginning of the title. The broad-reach papers are respectively a survey, a tutorial, and a design document for the popular \texttt{scikit-learn} library.}
  \label{tab:reach-exemplars}
  \small
  \setlength{\tabcolsep}{5pt}
  \begin{tabular}{@{}rrrllll@{}}
    \toprule
    $\chi$ & $\#$ & $f_{\times}$ &
    $Q$ & class & arXiv ID & title start \\
    \midrule
    \multirow{2}{*}{}
      20 &  1 & 0.00 & 0.40 & \texttt{cs.IT} & \texttt{1705.06350} & Wireless Information and Power Transfer over an AWGN channel\dots
 \\
      21 & 11 & 0.86 & 0.65 & \texttt{cs.LG} & \texttt{cs/0011033} & Web Mining Research: A Survey\\
    \addlinespace
    \multirow{2}{*}{}
      41 &  1 & 0.00 & 0.45 & \texttt{cs.IT} & \texttt{0905.3109} & Interference Channels with Source Cooperation \\
      38 & 14 & 0.76 & 0.65 & \texttt{cs.LG} & \texttt{1404.1100} & A Tutorial on Principal Component Analysis \\
    \addlinespace
    \multirow{2}{*}{}
      78 &  1 & 0.00 & 0.56 & \texttt{cs.CV} & \texttt{1803.09786} & Transferable Joint Attribute-Identity Deep Learning for Unsupervised\dots \\
      78 & 18 & 0.55 & 0.72 & \texttt{cs.LG} & \texttt{1309.0238} & API design for machine learning software: experiences from the scikit-\dots
 \\
    \bottomrule
  \end{tabular}
\end{table*}

Storing the $n\times n$ distance matrix for $n=169343$ requires roughly 57 GB even in half precision, so the dense evaluation of $Q$ is impractical at this scale. By comparison, a bit-packed sign sketch requires roughly a thousand times less storage for $N = 2500$. Using the approach of \S \ref{sec:geodesic} with $N = 2500$ random directions took just 3.76 seconds for the one-time sketch construction on a MacBook Pro (M3 Max), while all $n$ entropies were computed in 18.6 seconds (this was fast because of sparsity). By comparison, just loading the dataset took 1.82 seconds. (Virtually identical results were obtained even faster for $N = 100$.)

\subsection{A larger-scale example}\label{sec:larger}

A second example on the \texttt{ogbn-mag} dataset of $n = 736389$ papers \cite{hu2020open} exhibits dispersion across multiple levels in Figure \ref{fig:families} and extremely concentrated and dispersed fields and institutions in Figure \ref{fig:extremes}. Dispersion across multiple levels of aggregation (e.g., papers, authors, coauthors, fields, and institutions) in Figure \ref{fig:families} is computed very economically from a single sketch by placing uniform distributions on the relevant supports. The single large matrix multiplication involved took under two seconds.

\begin{figure}[htbp]
  \centering
  \includegraphics[trim = 0mm 0mm 0mm 0mm, clip, width=\columnwidth,keepaspectratio]{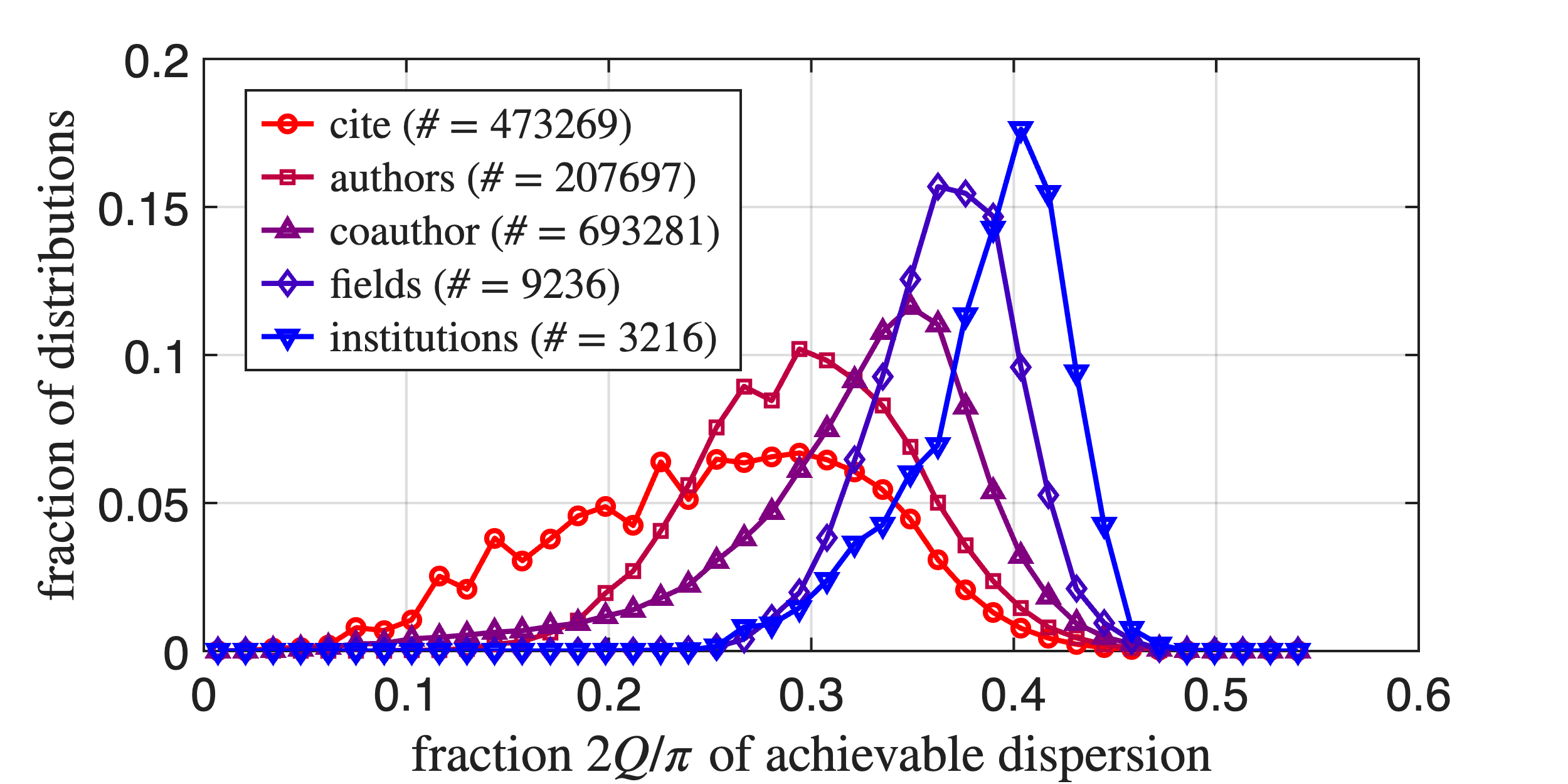}
  \caption{Quadratic entropy/dispersion increases with the level of aggregation.}
  \label{fig:families}
\end{figure}

\begin{figure}[htbp]
  \centering
  \includegraphics[trim = 0mm 0mm 0mm 0mm, clip, width=\columnwidth,keepaspectratio]{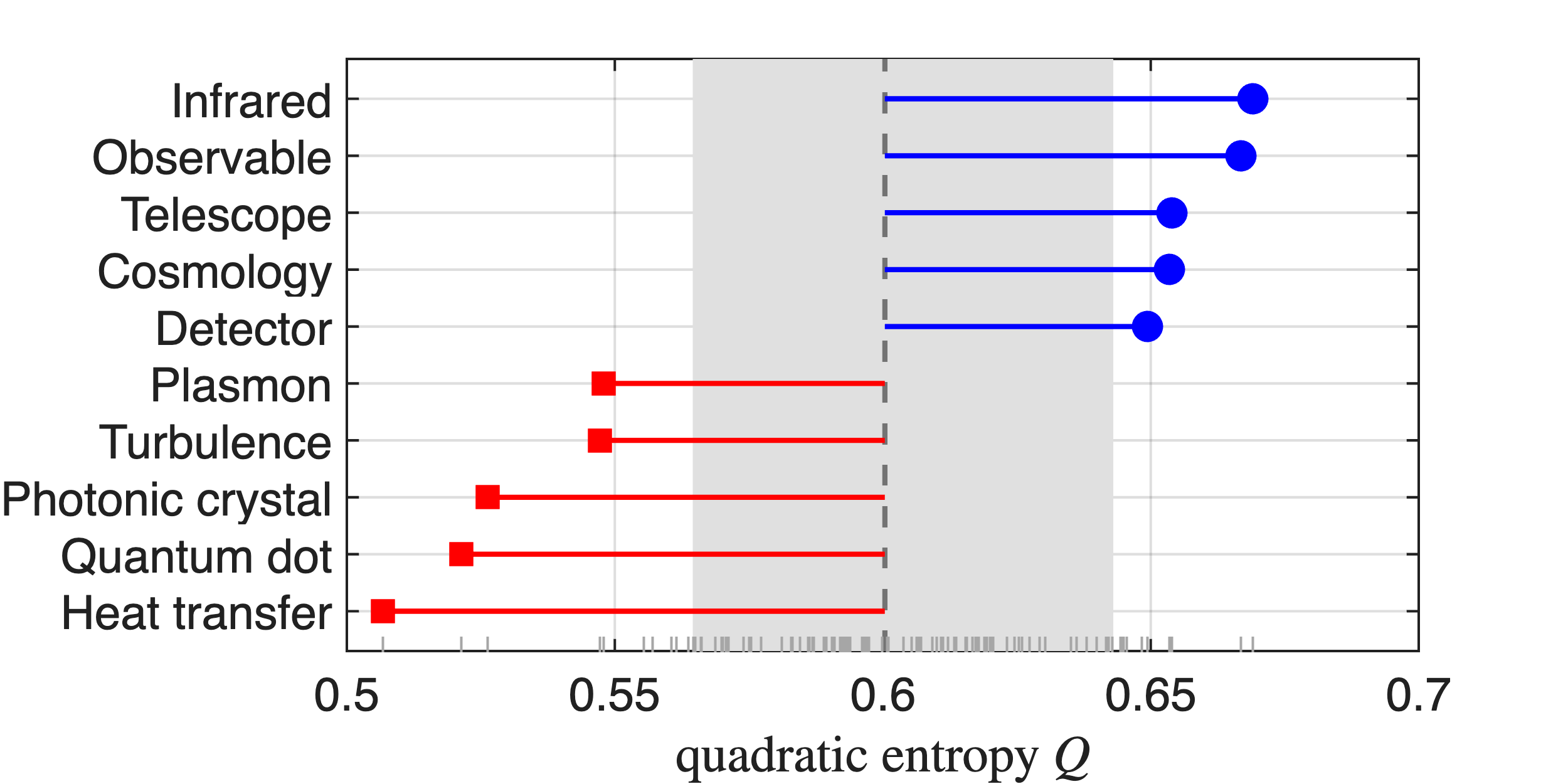}
  \includegraphics[trim = 0mm 0mm 0mm 0mm, clip, width=\columnwidth,keepaspectratio]{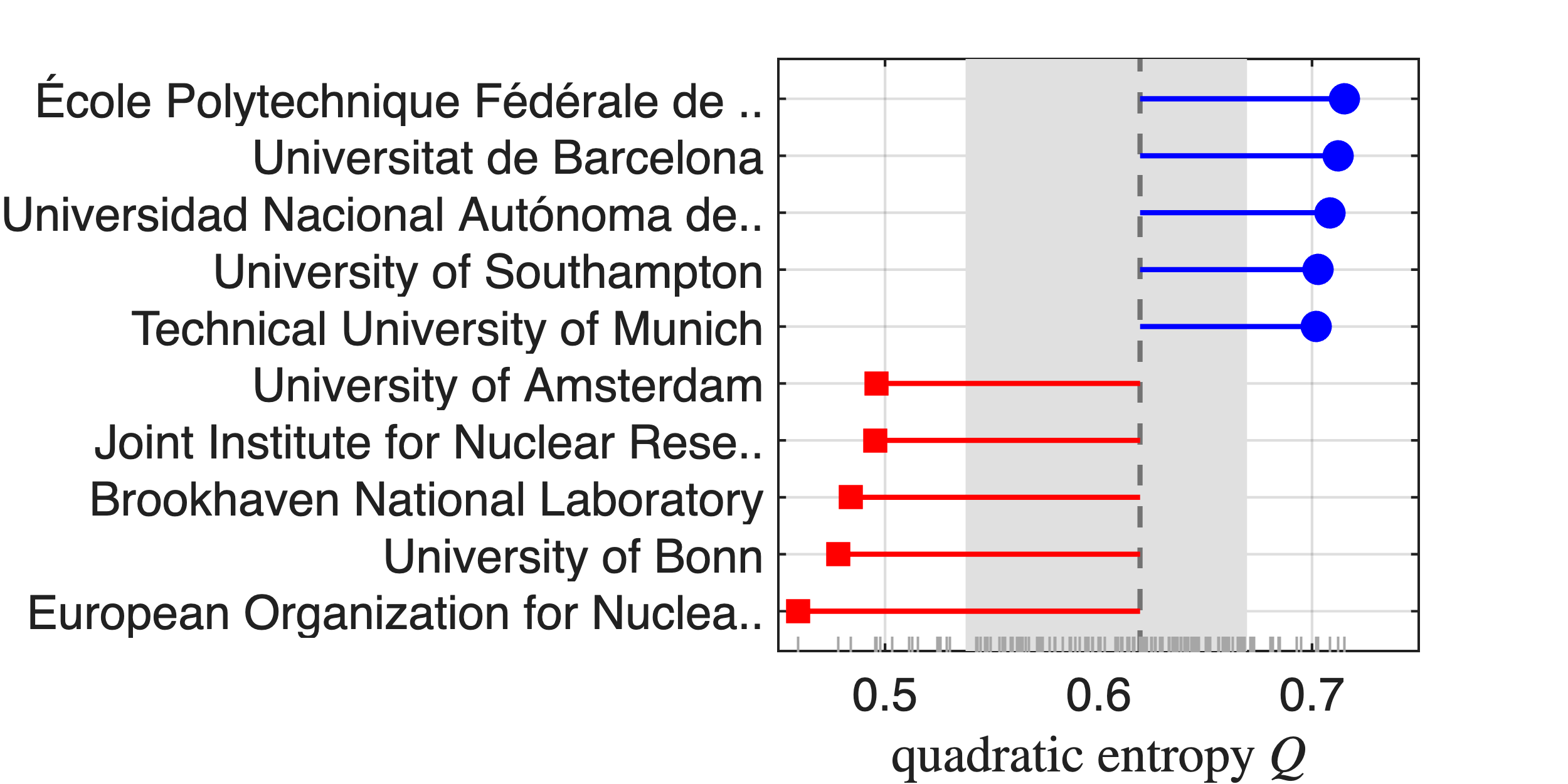}
  \caption{Extremes of {\color{red}concentration (red)} and {\color{blue}dispersion (blue)} among (top) fields at MAG concept hierarchy level 2 and (bottom) institutions, both with 5000 or more papers, with labels provided by OpenAlex \cite{priem2022openalex}. Concentration appears to be correlated with specialization (at least at extremes).}
  \label{fig:extremes}
\end{figure}

\subsection*{AI use statement}

In this work, I used generative AI tools (specifically, Claude Opus 5) to help develop conceptual frameworks, formulate mathematical claims, provide critical ingredients for proving mathematical claims, assist in the writing of proofs, propose or refine hypotheses, design or provide feedback on research methodology or experiments, implement methods, support qualitative and thematic data analysis, and interpret results. 

I have not used generative AI tools to generate synthetic data sets, assist with translation, or clean and reformat datasets.

Additionally, I used generative AI tools to create or modify scientific figures or images, suggest experimental parameters, create or edit software code, draft parts of a research paper, summarize or analyse existing literature, discover research topics or identify gaps, brainstorming, sourcing/searching for information, identify relevant literature, and format references.

I have reviewed all AI-assisted work, by rewriting proofs from scratch, introspecting and plotting Python code outputs in MATLAB, verifying other feedback and products at time of production, and performing independent analyses of data and results. 

In particular, all of the writing and responsibility is mine.

%

\bibliographystyle{plain}
\bibliography{qe}

\clearpage
\appendix
\thispagestyle{empty}

\onecolumn
\aistatstitle{Efficient quadratic entropy with distance sketches: \\
Supplementary Materials}

\section{Proofs}\label{sec:proofs}

\subsection{Proof of Proposition \ref{pro:EuclideanSketch}}

    \begin{proof}
        By standard algebra of Gaussian random variables (see, e.g., \S 14.3 of \cite{wasserman2004all}), $u^T x \sim \mathcal{N}(0,\|x\|^2)$. Using the substitution $t = z^2/2\| x \|^2$, we have
        \begin{align*}
        	\mathbb{E}_u|u^T x| & = \frac{2}{\sqrt{2 \pi \| x \|^2}} \int_0^\infty z e^{-z^2/2\|x\|^2} \ dz \\
    	& = \frac{2 \| x \|^2}{\sqrt{2 \pi \| x \|^2}} \int_0^\infty e^{-t} \ dt  = \sqrt{\frac{2}{\pi}} \cdot \| x \|,
        \end{align*}
        which establishes \eqref{eq:EuclideanSketch}. It follows that 
        \begin{align*}
        	p^T d p & = \sum\nolimits_{j,k} p_j p_k \| x^{(j)} - x^{(k)} \| \\
    	& = \sqrt{\frac{\pi}{2}} \cdot \mathbb{E}_u \sum\nolimits_{j,k} p_j p_k | u^T x^{(j)} - u^Tx^{(k)} |,
        \end{align*}
        which is precisely \eqref{eq:EuclideanSketchQuadraticDiversity}.
    \end{proof}

\subsection{Proof of Proposition \ref{pro:GiniMeanDifference}}

    \begin{proof}
        Because the $v_j$ are sorted,
        \begin{align}
            & \sum\nolimits_{j,k} w_j w_k |v_j - v_k| = 2\sum\nolimits_{j<k} w_j w_k (v_k - v_j) \nonumber \\
            & = 2\sum\nolimits_{j<k} w_j w_k v_k - 2\sum\nolimits_{j<k} w_j w_k v_j \nonumber \\
            & = 2 \sum\nolimits_k w_k v_k \sum\nolimits_{j = 1}^{k-1} w_j - 2\sum\nolimits_j w_j v_j \sum\nolimits_{k=j+1}^n w_k \nonumber \\
            & = 2\sum\nolimits_j w_j v_j \left ( \sum\nolimits_{k = 1}^{j-1} w_k - \sum\nolimits_{k=j+1}^n w_k \right ) \nonumber \\
            & = 2\sum\nolimits_j w_j v_j \left ( 2 \left [ \sum\nolimits_{k = 1}^{j-1} w_k \right ] + w_j - \sum\nolimits_{i=1}^n w_i \right ) \nonumber
        \end{align}
        where the last equality follows from $\sum_{k=j+1}^n w_k = (\sum\nolimits_{k=1}^n w_k) - w_j - \sum\nolimits_{k = 1}^{j-1} w_k$.
    \end{proof}

\subsection{Proof of Theorem \ref{thm:EuclideanEstimator}}
 
    \begin{proof}
    To start, write $x^{(jk)} := x^{(j)} - x^{(k)}$: now
    \begin{align*}
        & | q^\textnormal{euc}(u) - q^\textnormal{euc}(u') | \\
        & = \sqrt{\tfrac{\pi}{2}} \left | \sum\nolimits_{j,k} p_j p_k \left ( |u^T(x^{(jk)})| - |u'^T(x^{(jk)})| \right ) \right | \\
        & \le \sqrt{\tfrac{\pi}{2}} \sum\nolimits_{j,k} p_j p_k  \left | |u^T(x^{(jk)})| - |u'^T(x^{(jk)})| \right | \\
        & \le \sqrt{\tfrac{\pi}{2}} \sum\nolimits_{j,k} p_j p_k  \left | (u-u')^T(x^{(jk)}) \right | \\
        & \le \| u - u' \| \cdot \sqrt{\tfrac{\pi}{2}} \sum\nolimits_{j,k} p_j p_k \left \| x^{(jk)} \right \| \\
        & = \| u - u' \| \cdot \sqrt{\tfrac{\pi}{2}} \cdot  Q.
    \end{align*}
    where the second inequality is from the reverse triangle inequality $| \| a \| - \| b \| | \le \| a - b \|$ and the third is Cauchy-Schwarz. So $q^\textnormal{euc}$ is Lipschitz with constant $\le \sqrt{\frac{\pi}{2}} \cdot Q$. 
    
    The Lipschitz property lifts to $\bar Q_N^\textnormal{euc}(u^{(1)},\dots u^{(N)}) := \frac{1}{N} \sum_{i=1}^N q^\textnormal{euc}(u^{(i)})$:
    \begin{align*}
        & \left | \bar Q_N^\textnormal{euc}(u^{(1)},\dots u^{(N)}) - \bar Q_N^\textnormal{euc}(u'^{(1)},\dots u'^{(N)}) \right | \\
        & \le \frac{1}{N} \sum_{i=1}^N \left |q^\textnormal{euc}(u^{(i)})-q^\textnormal{euc}(u'^{(i)}) \right| \\
        & \le \frac{Q}{N} \sqrt{\frac{\pi}{2}} \sum_{i=1}^N \left \|u^{(i)}-u'^{(i)} \right \|  \\
        & \le Q \sqrt{\frac{\pi}{2N}} \left ( \sum_{i=1}^N \left \|u^{(i)}-u'^{(i)} \right \|^2 \right )^{1/2} 
    \end{align*}
    where the last inequality is Cauchy-Schwarz. We abbreviate this as 
    $$\left | \bar Q_N^\textnormal{euc}(U) - \bar Q_N^\textnormal{euc}(U') \right | \le Q \sqrt{\frac{\pi}{2N}} \cdot \left \| U- U' \right \|$$
    where we can regard $U \sim \mathcal{N}(0,I_{Nm})$. Since $\mathbb{E}_U \bar Q_N^\textnormal{euc} = Q$, the Gaussian concentration inequality (Theorem 5.6 of \cite{boucheron2013concentration}) now yields
    $$\mathbb{P} \left ( \left | \bar Q_N^\textnormal{euc} - Q \right | \ge t  \right ) \le 2\exp(-Nt^2/\pi Q^2).$$
    Substituting $\varepsilon := t/Q$ yields \eqref{eq:EuclideanEstimatorTail}.
    
    Meanwhile, for $v \in \mathbb{R}^m$ and $u \sim \mathcal{N}(0,I_m)$ we have $u^Tv \sim \mathcal{N}(0,\|v\|^2)$, so $\mathbb{E}|u^Tv| = \sqrt{2/\pi} \cdot \|v\|$ by \eqref{eq:EuclideanSketch} while $\mathbb{E}[(u^Tv)^2] = \|v\|^2$. This gives $\mathrm{Var}(|u^Tv|) = (1-2/\pi)\|v\|^2$. Since $\textnormal{Var}(\sum_j \beta_j Y^{(j)}) \le \left ( \sum_j |\beta_j|\sqrt{\textnormal{Var}(Y^{(j)})} \right )^2$ by Minkowski's $L^2$ inequality, 
    \begin{align*}
    \textnormal{Var}(q^\textnormal{euc}) & = \textnormal{Var} \left ( \sqrt{\frac{\pi}{2}} \cdot \sum_{j,k} p_j p_k \left | u^T (x^{(j)}-x^{(k)}) \right | \right ) \\
    & \le \frac{\pi}{2} \left ( \sum_{j,k} p_j p_k \sqrt{\textnormal{Var}\left | u^T (x^{(j)}-x^{(k)}) \right |} \right )^2 \\
    & = \frac{\pi}{2} \left ( \sum_{j,k} p_j p_k d_{jk} \sqrt{1-\frac{2}{\pi}} \right )^2 \\
    & = \left ( \frac{\pi}{2}-1 \right ) Q^2
    \end{align*}
    and $\mathrm{Var}(\bar Q^\textnormal{euc}_N) = \mathrm{Var}(q^\textnormal{euc})/N$.
    \end{proof}
    
\subsection{Proof of Proposition \ref{pro:sphereSignEmbedding}}

    \begin{proof}
        $s_j \in \{\pm 1\}$, $s_j^2 = 1$ and $(s_j-s_k)^2 = 2(1-s_js_k)$, so 
        $\| \psi^{(j)} - \psi^{(k)} \|_{L^2(S^{m-1})}^2 = \frac{\pi}{4} \mathbb{E}_u([s_j-s_k]^2) = \frac{\pi}{2}(1-\mathbb{E}_u[s_js _k]) = \theta_{jk}$,
        where the last equality is \eqref{eq:geodesic1}.
    \end{proof}
    
\subsection{Proof of Theorem \ref{thm:geodesicSketchQuadraticDiversity}}

    \begin{proof}\footnote{Alternatively, use \eqref{eq:variance} on $\sum_j p_j \| \psi^{(j)} \|^2 = \frac{\pi}{4} \sum_j \mathbb{E}_u (s_j^2) = \frac{\pi}{4}$ and $\| \sum_j p_j \psi^{(j)} \|^2 = \frac{\pi}{4} \mathbb{E}_u (s^2)$.
    }
    \begin{align}
        p^T d p & = \sum\nolimits_{j,k} p_j p_k \theta_{jk} \nonumber \\
        & = \frac{\pi}{2} \sum\nolimits_{j,k} p_j p_k(1-\mathbb{E}_u [s_j s_k]) \nonumber \\
        & = \frac{\pi}{2} (1-\mathbb{E}_u [\sum\nolimits_{j,k} p_j p_k s_j s_k]) \nonumber \\
        & = \frac{\pi}{2}(1-\mathbb{E}_u[s^2]). \nonumber
    \end{align}
    \end{proof}    
    
\subsection{Proof of Theorem \ref{thm:geodesicEstimatorAbsolute}}
 
    \begin{proof}
    $q^\textnormal{geo} \in [ 0,\pi/2 ]$, so Hoeffding's inequality (Theorem 2.8 of \cite{boucheron2013concentration}) immediately gives \eqref{eq:geodesicEstimatorTail} and also $\text{Var}(q^\textnormal{geo}) \le \pi^2/16$. Meanwhile, 
    $$\text{Var}(\bar Q_N^\textnormal{geo}) = \frac{1}{N^2} \text{Var} \left ( \sum\nolimits_{i=1}^N q^\textnormal{geo}(u^{(i)})\right ) = \frac{1}{N} \text{Var}(q^\textnormal{geo}).$$ 
    Combining this with $\text{Var}(q^\textnormal{geo}) \le \pi^2/16$ yields \eqref{eq:geodesicEstimatorVariance}.
    \end{proof}

\subsection{Proof of Proposition \ref{prop:sharpvar}}

    \begin{proof}
    Since $s_j(u) \in \{\pm 1\}$, the convex combination $s(u) \in [-1,1]$, so $s^2 \in [0,1]$ and therefore $\mathbb{E}_u(s^4) \le \mathbb{E}_u(s^2)$. It follows that 
    $$\mathrm{Var}(s^2) = \mathbb{E}_u(s^4) - [\mathbb{E}_u(s^2)]^2 \le \mathbb{E}_u(s^2) \cdot [1 - \mathbb{E}_u(s^2)].$$
    Meanwhile, $\mathbb{E}_u(s^2) = 1 - 2Q/\pi$ by Theorem \ref{thm:geodesicSketchQuadraticDiversity}, so
    $$\mathrm{Var}(s^2) \le (1 - 2Q/\pi) \cdot 2Q/\pi.$$
    Since $q^{\textnormal{geo}}(u) = \frac{\pi}{2} (1 - s(u)^2)$, we have 
    $$\mathrm{Var}(q^{\textnormal{geo}}) = \frac{\pi^2}{4}\mathrm{Var}(s^2) \le \frac{\pi}{2}Q-Q^2.$$
    The second inequality in the proposition follows because $\mu(1-\mu) \le 1/4$ for $\mu \in [0,1]$. (Alternatively, note that $\frac{\pi}{2}Q - Q^2$ is maximized over $Q \in [0, \pi/2]$ at $Q = \pi/4$, where it equals $\pi^2/16$.)
    \end{proof}

\subsection{Proof of Lemma \ref{lem:bern}}

    \begin{proof}
    The overarching idea is to apply Theorem \ref{thm:blm} with $\nu := N\sigma^2$ and $c := R/3$. Now \eqref{eq:bcond1} is satisfied. To see that \eqref{eq:bcond2} is satisfied, fix an integer $\ell \ge 3$. Now $(\zeta_i)_+^{\ell} \le R^{\ell-2}\zeta_i^2$ since $\ell - 2 \ge 1$ and $0 \le (\zeta_i)_+ \le R$. Taking expectations and applying \eqref{eq:bcond1} gives that $N \mathbb{E} \left[(\zeta_i)_+^{\ell}\right] \le R^{\ell-2}\nu$, so \eqref{eq:bcond2} follows once $R^{\ell-2}\nu \le \frac{\ell!}{2}\nu\left(R/3\right)^{\ell-2}$. Equivalently, \eqref{eq:bcond2} follows once 
    $2\cdot 3^{\ell-2} \le \ell!$. This inequality holds for $\ell \ge 3$, with equality at $\ell = 3$ (which is where the constant term in $c = R/3$ comes from).
    
    It now remains to pick a good value of $\tau$ and produce a two-sided bound. Towards the first of these goals, Theorem \ref{thm:blm} controls $S := \sum_{i=1}^{N}\zeta_i$ at the $\tau$-dependent level $\sqrt{2\nu\tau} + c\tau$. To control it at the prescribed level $Nr$, set
    $$\tau^{*} := \frac{N r^2}{2\sigma^2 + \frac{2R}{3}r} = \frac{(Nr)^2}{2\left(\nu + cNr\right)};$$
    $$\vartheta := \frac{\nu}{\nu + cNr} = \frac{\sigma^2}{\sigma^2 + \frac{R}{3}r} \in (0,1],$$
    so that $cNr/(\nu + cNr) = 1 - \vartheta$, and hence $\sqrt{2\nu\tau^{*}} = Nr\sqrt{\vartheta}$ and $c\tau^{*} = \frac{Nr}{2}\left(1 - \vartheta\right)$. Completing the square,
    \begin{align*}
    \sqrt{2\nu\tau^{*}} + c\tau^{*} 
      & = Nr \left( \sqrt{\vartheta} + \frac{1-\vartheta}{2} \right) \\
      & = Nr \left( 1 - \frac{(\sqrt{\vartheta}-1)^2}{2} \right)
      \le Nr.
      \end{align*}
    Thus $S \ge Nr \Rightarrow S \ge \sqrt{2\nu\tau^{*}} + c\tau^{*}$, so \eqref{eq:bernstein} applies at $\tau = \tau^{*}$. The result follows by a union bound and sign symmetry of the hypothesis $|\zeta_i| \le R$.
    \end{proof}

\subsection{Proof of Theorem \ref{thm:geodesicEstimatorRelative}}

    \begin{proof}
    Set $\zeta_i := q^{\textnormal{geo}}(u^{(i)}) - Q$, so the $\zeta_i$ are IID and centered with $\frac{1}{N}\sum_{i=1}^{N}\zeta_i = \bar{Q}^{\textnormal{geo}}_N - Q$. We supply the two inputs $R$ and $\sigma^2$ that Lemma \ref{lem:bern} requires. Since $s_j(u) \in \{\pm 1\}$, the convex combination $s(u) = \sum_j p_j s_j(u)$ is in $[-1,1]$, so $|s(u)| \le 1$ and $q^{\textnormal{geo}}(u) = \frac{\pi}{2}\left(1 - s(u)^2\right) \in [0,\pi/2]$. Similarly, $Q = \mathbb{E}_u q^{\textnormal{geo}} \in [0,\pi/2]$, so $|\zeta_i| \le \pi/2 =: R$. Meanwhile, Proposition \ref{prop:sharpvar} gives $\mathbb{E}[\zeta_i^2] = \mathrm{Var}(q^{\textnormal{geo}}) \le \pi Q/2 - Q^2 =: \sigma^2$.
    
    Now with $r = \varepsilon Q$,
    \begin{align*}
    \frac{N r^2}{2\sigma^2 + 2rR/3} 
    & = \frac{N\varepsilon^2 Q^2}
           {2\left(\pi Q/2 - Q^2\right) + \pi\varepsilon Q/3} \\
    & = \frac{N\varepsilon^2 Q}{\pi - 2Q + \pi\varepsilon/3}
    \ge \frac{N\varepsilon^2 Q}{\pi\left(1 + \varepsilon/3\right)}.
    \end{align*}
    Meanwhile, $| \tfrac{1}{N} \textstyle\sum_{i=1}^{N}\zeta_i | \ge r$ is the same as $\left|\bar{Q}^{\textnormal{geo}}_N/Q - 1\right| \ge \varepsilon$. These yield \eqref{eq:geobern}.
    \end{proof}

\subsection{Proof of Lemma \ref{lem:uniformAmortizedEuclidean}}

    \begin{proof}
    Write $g_i := u^{(i)T}(x^{(j)}-x^{(k)})/d_{jk}$, so that $(g_1,\dots,g_N) \sim \mathcal{N}(0,I_N)$ are IID. Since by Proposition \ref{pro:EuclideanSketch}, $d_{jk} = \sqrt{\frac{\pi}{2}} \cdot \mathbb{E}_u | u^T (x^{(j)} - x^{(k)}) |$ for $u \sim \mathcal{N}(0,I_m)$, we have
    \begin{align*}
    (\bar d_N)_{jk} & = \sqrt{\frac{\pi}{2}} \frac{1}{N} \sum_{i=1}^N \left | u^{(i)T}(x^{(j)}-x^{(k)}) \right | \\
    & = \sqrt{\frac{\pi}{2}} d_{jk} \cdot \frac{1}{N} \sum_{i=1}^N \left | g_i \right | \rightarrow d_{jk}.
    \end{align*}
    Further writing here $F(g) \equiv F_N(g_1,\dots,g_N) := \frac{1}{N} \sum_{i=1}^N |g_i|$, we have that 
    $$|F(g)-F(g')| \le \frac{1}{N} \sum_{i=1}^N |g_i-g'_i| \le \frac{1}{\sqrt{N}} \| g-g' \|,$$
    where the second inequality is from Cauchy-Schwarz. The Gaussian concentration inequality (Theorem 5.6 of \cite{boucheron2013concentration}) now yields $\mathbb{P}(|F-\mathbb{E}F| \ge t) \le 2 \exp(-Nt^2/2)$. Setting $\varepsilon := t \cdot \sqrt{\pi/2}$ and using $(\bar d_N)_{jk}/d_{jk} = \sqrt{\pi/2} \cdot F \rightarrow 1$, we get in turn
    $$\mathbb{P}\left ( \left | \frac{(\bar d_N)_{jk}}{d_{jk}}-1 \right | \ge \varepsilon \right ) 
    \le 2\exp \left ( - \frac{N \varepsilon^2}{\pi} \right ).$$
    
    Now a union bound over the $\binom{n}{2}$ unordered pairs $(j,k)$ gives that
    $$\mathbb{P}\left ( \max_{j,k} \left | \frac{(\bar d_N)_{jk}}{d_{jk}} - 1 \right | \ge \varepsilon \right ) 
    \le 2 \binom{n}{2} \cdot \exp \left ( - \frac{N \varepsilon^2}{\pi} \right ),$$
    which also implies that 
    \begin{align*}
    \left | \bar Q^\textnormal{euc}_N(p) - Q(p) \right | & = \left | \sum_{j,k} p_j p_k \left [ (\bar d_N)_{jk} - d_{jk} \right ] \right | \\
    & \le \varepsilon \sum_{j,k} p_j p_k d_{jk} = \varepsilon Q(p),
    \end{align*}
    so
    $$\mathbb{P}\left ( \left | \frac{\bar Q^\textnormal{euc}_N(p)}{Q(p)} - 1 \right | \ge \varepsilon \right ) 
    \le 2 \binom{n}{2} \cdot \exp \left ( - \frac{N \varepsilon^2}{\pi} \right ).$$
    The result follows by solving for $N$.
    \end{proof}

\subsection{Proof of Lemma \ref{lem:uniformAmortizedGeodesic}}

    \begin{proof}
        The random variable $\frac{\pi}{2} (1-s_j(u)s_k(u))$ takes values in $\{0,\pi\}$ with mean $\theta_{jk}$. Using notation along the lines of the preceding proof, Hoeffding's inequality yields $\mathbb{P}(|(\bar \theta_N)_{jk}-\theta_{jk} | \ge \varepsilon) \le 2 \exp(-2N\varepsilon^2/\pi^2)$. A union bound gives in turn 
        $$\mathbb{P} \left ( \max_{j,k} \left |(\bar \theta_N)_{jk}-\theta_{jk} \right | \ge \varepsilon \right ) \le 2 \binom{n}{2} \cdot \exp \left (-\frac{2N\varepsilon^2}{\pi^2} \right )$$
        which is bounded above by $\delta$. 
        Finally, $\left |\bar Q_N^\textnormal{geo}-Q \right |  \le \sum_{j,k} p_j p_k \left | (\bar \theta_N)_{jk} - \theta_{jk} \right |$, so if $\max_{j \ne k} \left |(\bar \theta_N)_{jk}-\theta_{jk} \right | \le \varepsilon$, then $\left |\bar Q^\textnormal{geo}_N-Q \right | \le \varepsilon \sum_{j\ne k}p_jp_k = \varepsilon \cdot (1-\|p\|^2) \le \varepsilon$.
    \end{proof}

\end{document}